\documentclass{article}
\usepackage[utf8]{inputenc}
\usepackage{authblk}
\usepackage[utf8]{inputenc} 
\usepackage[T1]{fontenc}    
\usepackage{booktabs}       

\usepackage{amsmath,amsfonts,amsthm,amssymb}
\usepackage[numbers,sort]{natbib}
\usepackage[backref,pdfpagelabels]{hyperref}
\usepackage{mathtools}
\usepackage{cleveref}
\usepackage{fullpage}

\newtheorem{theorem}{Theorem}[section]
\newtheorem{lemma}[theorem]{Lemma}

\newtheorem{corollary}[theorem]{Corollary}

\def\E{\mathbb{E}}
\def\R{\mathbb{R}}
\def\A{\mathcal{A}}

\DeclarePairedDelimiter{\abs}{\lvert}{\rvert}
\DeclarePairedDelimiter{\norm}{\|}{\|}

\DeclarePairedDelimiter{\set}{ \{ }{ \} }
\newcommand{\bkt}[2]{\langle #1, #2 \rangle}

\newcommand{\eps}{\varepsilon}

\newcommand{\poly}{\mathrm{poly}}
\newcommand{\Vol}{\mathsf{Vol}}
\newcommand{\cg}{\mathsf{cg}}
\newcommand{\length}{\mathsf{length}}
\newcommand{\Conv}{\mathsf{Conv}}
\newcommand{\B}{\mathsf{B}}
\newcommand{\loss}{\mathsf{loss}}

\newcommand{\width}{\mathsf{width}}
\newcommand{\BR}{\mathsf{BR}}
\newcommand{\ball}{\B}
\newcommand{\Sph}{\mathbb{S}}

\newcommand{\X}{\ensuremath{\mathcal{X}}}

\DeclareMathOperator*{\argmax}{arg\,max}

\title{Contextual Recommendations and Low-Regret Cutting-Plane Algorithms}
\author{Sreenivas Gollapudi}
\author{ Guru Guruganesh}
\author{Kostas Kollias}
\author{Pasin Manurangsi}
\author{Renato Paes Leme}
\author{ Jon Schneider }
\affil{Google Research}

\begin{document}

\maketitle

\begin{abstract}
 
 We consider the following variant of contextual linear bandits motivated by routing applications in navigational engines  and recommendation systems. We wish to learn a hidden $d$-dimensional value $w^*$. Every round, we are presented with a subset $\mathcal{X}_t \subseteq \mathbb{R}^d$ of possible actions. If we choose (i.e. recommend to the user) action $x_t$, we obtain utility $\langle x_t, w^* \rangle$ but only learn the identity of the best action $\arg\max_{x \in \X_t} \langle x, w^* \rangle$.

We design algorithms for this problem which achieve regret $O(d\log T)$ and $\exp(O(d \log d))$. To accomplish this, we design novel cutting-plane algorithms with low “regret” -- the total distance between the true point $w^*$ and the hyperplanes the separation oracle returns. 

We also consider the variant where we are allowed to provide a list of several recommendations. In this variant, we give an algorithm with $O(d^2 \log d)$ regret and list size $\poly(d)$. Finally, we construct nearly tight algorithms for a weaker variant of this problem where the learner only learns the identity of an action that is better than the recommendation. Our results rely on new algorithmic techniques in convex geometry (including a variant of Steiner’s formula for the centroid of a convex set) which may be of independent interest. 
\end{abstract}

\section{Introduction}

Consider the following problem faced by a geographical query service (e.g. Google Maps). When a user searches for a path between two endpoints, the service must return one route out of a set of possible routes. Each route has a multidimensional set of features associated with it, such as (i) travel time, (ii) amount of traffic, (iii) how many turns it has, (iv) total distance, etc.  The service must recommend one route to the user, but doesn’t a priori know how the user values these features relative to one another. However, when the service recommends a route, the service can observe some feedback from the user: whether or not the user followed the recommended route (and if not, which route the user ended up taking). How can the service use this feedback to learn the user’s preferences over time?

Similar problems are faced by recommendation systems in general, where every round a user arrives accompanied by some contextual information (e.g. their current search query, recent activity, etc.), the system makes a recommendation to the user, and the system can observe the eventual action (e.g. the purchase of a specific item) by the user. These problems can be viewed as specific cases of a variant of linear contextual bandits that we term \textit{contextual recommendation}. 

In contextual recommendation, there is a hidden vector $w^* \in \R^d$ (e.g. representing the values of the user for different features) that is unknown to the learner. Every round $t$ (for $T$ rounds), the learner is presented with an adversarially chosen (and potentially very large) set of possible actions $\mathcal{X}_t$. Each element $x_t$ of $\mathcal{X}_t$ is also an element of $\R^d$ (visible to the learner); playing action $x_t$ results in the learner receiving a reward of $\bkt{ x_t}{ w^*}$. The learner wishes to incur low regret compared to the best possible strategy in hindsight -- i.e. the learner wishes to minimize

\begin{equation}\label{eq:contextual-selection-regret}
\mathrm{Reg} = \sum_{t=1}^{T} \left(\langle x^*_t, w^*\rangle - \langle x_t, w^*\rangle\right),
\end{equation}

\noindent
where $x^*_t = \arg\max_{x \in \mathcal{X}_t} \bkt{ x}{ w^*}$ is the best possible action at time $t$. In our geographical query example, this regret corresponds to the difference between the utility of a user that always blindly follows our recommendation and the utility of a user that always chooses the optimal route. 

Thus far this agrees with the usual set-up for contextual linear bandits (see e.g. \cite{chu2011contextual}). Where contextual recommendation differs from this is in the feedback available to the learner: whereas classically in contextual linear bandits the learner learns (a possibly noisy version of) the reward they receive each round, in contextual recommendation the learner instead learns \textit{the identity of the best arm $x^*_t$}. This altered feedback makes it difficult to apply existing algorithms for linear contextual bandits. In particular, algorithms like LINUCB and LIN-Rel \cite{chu2011contextual, auer2002using} all require estimates of $\langle x_t, w^* \rangle$ in order to learn $w^*$ over time, and our feedback prevents us from obtaining any such absolute estimates. 

In this paper we design low-regret algorithms for this problem. We present two algorithms for this problem: one with regret $O(d\log T)$ and one with regret $\exp(O(d\log d))$ (Theorems \ref{thm:expdlogd} and \ref{thm:dlogt}). Note that both regret guarantees are independent of the number of offered actions $|\mathcal{X}_t|$ (the latter even being independent of the time horizon $T$). Moreover both of these algorithms are efficiently implementable given an efficient procedure for optimizing a linear function over the sets $\mathcal{X}_t$. This condition holds e.g. in the example of recommending shortest paths that we discussed earlier.

In addition to this, we consider two natural extensions of contextual recommendation where the learner is allowed to recommend a bounded subset of actions instead of just a single action (as is often the case in practice). In the first variant, which we call \emph{list contextual recommendation}, each round the learner recommends a set of at most $L$ (for some fixed $L$) actions to the learner. The learner still observes the user's best action each round, but the loss of the learner is now the difference between the utility of the best action for the user and the best action offered by the learner (capturing the difference in utility between a user playing an optimal action and a user that always chooses the best action the learner offers).

In list contextual recommendation, the learner has the power to cover multiple different user preferences simultaneously (e.g. presenting the user with the best route for various different measures). We show how to use this power to construct an algorithm for the learner which offers $\poly(d)$ actions each round and obtain a total regret of $O(\poly(d))$. 

In the second variant, we relax an
assumption of both previous models: that the user will always choose their best possible action (and hence that we will observe their best possible action). To relax this assumption, we also consider the following weaker version of contextual recommendation we call \textit{local contextual recommendation}.

In this problem, the learner again recommends a set of at most $L$ actions to the learner (for some $L > 1$)\footnote{Unlike in the previous two variants, it is important in local contextual recommendation that $L > 1$; if $L=1$ then the user can simply report the action the learner recommended and the learner receives no meaningful feedback.}. The user then chooses an action which is at least as good as the best action in our list, and we observe this action. In other words, we assume the learner at least looks at all the options we offer, so if they choose an external option, it must be better than any offered option (but not necessarily the global optimum). Our regret in this case is the difference between the total utility of a learner that always follows the best recommendation in our list and the total utility of a learner that always plays their optimal action\footnote{In fact, our algorithms all work for a slightly stronger notion of regret, where the benchmark is the utility of a learner that always follows the \textit{first} (i.e. a specifically chosen) recommendation on our list. With this notion of regret, contextual recommendation reduces to local contextual recommendation with $L = \max |\mathcal{X}_t|$.}. 

Let $A = \max_{t} |\mathcal{X}_t|$ be a bound on the total number of actions offered in any round, and let $\gamma = A/(L-1)$. Via a simple reduction to contextual recommendation, we construct algorithms for local contextual recommendation with regret $O(\gamma d\log T)$ and $\gamma\exp(O(d\log d))$. We further show that the first bound is ``nearly tight'' (up to $\poly(d)$ factors) in some regimes; in particular, we demonstrate an instance where $L = 2$ and $K = 2^{\Omega(d)}$ where any algorithm must incur regret at least $\min(2^{\Omega(d)}, \Omega(T))$ (\Cref{thm:locallb}).

\subsection{Low-regret cutting plane methods and contextual search}\label{sec:introlrcp}

To design these low-regret algorithms, we reduce the problem of contextual recommendation to a geometric online learning problem (potentially of independent interest). We present two different (but equivalent) viewpoints on this problem: one motivated by designing separation-oracle-based algorithms for convex optimization, and the other by contextual search.

\subsubsection{Separation oracles and cutting-plane methods}

Separation oracle methods (or ``cutting-plane methods'') are an incredibly well-studied class of algorithms for linear and convex optimization. For our purposes, it will be convenient to describe cutting-plane methods as follows. 

Let $\ball = \{w \in \R^d \mid\, \norm{w} \leq 1\}$ be the unit ball in $\R^d$. We are searching for a hidden point $w^* \in \ball$. Every round we can choose a point $p_t \in \ball$ and submit this point to a \textit{separation oracle}. The separation oracle then returns a half-space separating $p_t$ from $w^*$; in particular, the oracle returns a direction $v_t$ such that $\langle w^*, v_t \rangle \geq \langle p_t, v_t\rangle$. 

Traditionally, cutting-plane algorithms have been developed to minimize the number of calls to the separation oracle until the oracle returns a hyperplane that passes within some distance $\delta$ of $w^*$. For example, the ellipsoid method (which always queries the center of the currently-maintained ellipse) has the guarantee that it makes at most $O(d^2 \log 1/\delta)$ oracle queries before finding such a hyperplane.

In our setting, instead of trying to minimize the number of separation oracle queries before finding a ``close'' hyperplane, we would like to minimize the total (over all $T$ rounds) distance between the returned hyperplanes and the hidden point $w^*$. That is, we would like to minimize the expression

\begin{equation}\label{eq:cutting-plane-reg}
\mathrm{Reg}'= \sum_{t=1}^{T} \left( \langle w^*, v_t \rangle - \langle p_t, v_t \rangle\right).
\end{equation}

Due to the similarity between \eqref{eq:cutting-plane-reg} and \eqref{eq:contextual-selection-regret}, we call this quantity the \textit{regret} of a cutting-plane algorithm. We show that, given any low-regret cutting-plane algorithm, there exists a low-regret algorithm for contextual recommendation.

\begin{theorem}[Restatement of Theorem \ref{thm:reduction}]\label{thm:intro_reduction}
Given a low-regret cutting-plane algorithm $\mathcal{A}$ with regret $\rho$, we can construct an $O(\rho)$-regret algorithm for contextual recommendation. 
\end{theorem}

This poses a natural question: what regret bounds are possible for cutting-plane methods? One might expect guarantees on existing cutting-plane algorithms to transfer over to regret bounds, but interestingly, this does not appear to be the case. In particular, most existing cutting-plane methods and analysis suffers from the following drawback: even if the method is likely to find a hyperplane within distance $\delta$ relatively quickly, there is no guarantee that subsequent calls to the oracle will return low-regret hyperplanes.

In this paper, we will show how to design low-regret cutting-plane methods. Although our final algorithms will bear some resemblance to existing cutting-plane algorithms (e.g. some involve cutting through the center-of-gravity of some convex set), our analysis will instead build off more recent work on the problem of \textit{contextual search}.

\subsubsection{Contextual search}

Contextual search is an online learning problem initially motivated by applications in pricing \cite{paesleme2018contextual}. The basic form of contextual search can be described as follows. As with the previously mentioned problems, there is a hidden vector $w^* \in [0, 1]^d$ that we wish to learn over time. Every round the adversary provides the learner with a vector $v_t$ (the ``context''). In response, the learner must guess the value of $\bkt{ v_t}{w^*}$, submitting a guess $y_t$. The learner then incurs a loss of $|\bkt{v_t}{w^*} - y_t|$ (the distance between their guess and the true value of the inner product), but only learns whether $\bkt{v_t}{w^*}$ is larger or smaller than their guess.

The problem of designing low-regret cutting plane methods can be interpreted as a  ``context-free'' variant of contextual search. In this variant, the learner is no longer provided the context $v_t$ at the beginning of each round, and instead of guessing the value of $\bkt{v_t}{ w^*}$, they are told to directly submit a guess $p_t$ for the point $w^*$. The context $v_t$ is then revealed to them \textit{after} they submit their guess, where they are then told whether $\bkt{p_t}{w^*}$ is larger or smaller than $\bkt{v_t}{w^*}$ and incur loss $|\bkt{v_t}{w^*} - \bkt{p_t}{w^*}|$. Note that this directly corresponds to querying a separation oracle with the point $p_t$, and the separation oracle returning either the halfspace $v_t$ (in the case that $\bkt{w^*}{v_t} \geq \bkt{w^*}{p_t}$) or the halfspace $-v_t$ (in the case that $\bkt{w^*}{v_t} \leq \bkt{w^*}{p_t}$). 

One advantage of this formulation is that (unlike in standard analyses of cutting-plane methods) the total loss in contextual search directly matches  the expression in \eqref{eq:cutting-plane-reg} for the regret of a cutting-plane method. In fact, were there to already exist an algorithm for contextual search which operated in the above manner -- guessing $\bkt{v_t}{w^*}$ by first approximating $w^*$ and then computing the inner product -- we could just apply this algorithm verbatim and get a cutting-plane method with the same regret bound. Unfortunately, both the algorithms of \cite{liu2021optimal} and \cite{paesleme2018contextual} explicitly require knowledge of the direction $v_t$.

This formulation also raises an interesting subtlety in the power of the separation oracle: specifically, whether the direction $v_t$ is fixed (up to sign) ahead of time or is allowed to depend on the point $p$. Specifically, we consider two different classes of separation oracles. For \textit{(strong) separation oracles}, the direction $v_t$ is allowed to freely depend on the point $p_t$ (as long as it is indeed true that $\bkt{w^*}{v_t} \geq \bkt{p_t}{v_t}$). For \textit{weak separation oracles}, the adversary fixes a direction $u_t$ at the beginning of the round, and then returns either $v_t = u_t$ or $v_t = -u_t$ (depending on the sign of $\bkt{w^* - p_t}{u_t}$). The strong variant is most natural when comparing to standard separation oracle guarantees (and is necessary for the reduction in Theorem \ref{thm:intro_reduction}), but for many standalone applications (especially those motivated by contextual search) the weak variant suffices. In addition, the same techniques we use to construct a cutting-plane algorithm for weak separation oracles will let us design low-regret algorithms for list contextual recommendation.


\subsection{Our results and techniques}

We design the following low-regret cutting-plane algorithms:

\begin{enumerate}
    \item An $\exp(O(d\log d))$-regret cutting-plane algorithm for strong separation oracles.
    \item An $O(d\log T)$-regret cutting-plane algorithm for strong separation oracles.
    \item An $O(\poly(d))$-regret cutting-plane algorithm for weak separation oracles. 
\end{enumerate}

All three algorithms are efficiently implementable (in $\poly(d, T)$ time). Through Theorem \ref{thm:intro_reduction}, points (1) and (2) immediately imply the algorithms with regret $\exp(O(d))$ and $O(d\log T)$ for contextual recommendation. Although we do not have a blackbox reduction from weak separation oracles to algorithms for list contextual recommendation, we show how to apply the same ideas in the algorithm in point (3) to construct an $O(d^2\log d)$-regret algorithm for list contextual recommendation with $L = \poly(d)$. 

To understand how these algorithms work, it is useful to have a high-level understanding of the algorithm of \cite{liu2021optimal} for contextual search. That algorithm relies on a multiscale potential function the authors call the \textit{Steiner potential}. The Steiner potential at scale $r$ is given by the expression $\Vol(K_t + r\B)$, where $K_t$ (the ``knowledge set'') is the current set of possibilities for the hidden point $w^*$, $\B$ is the unit ball, and addition denotes Minkowsi sum; in other words, this is the volume of the set of points within distance $r$ of $K_t$. The authors show that by choosing their guess $y_t$ carefully, they can decrease the $r$-scale Steiner potential (for some $r$ roughly proportional to the width of $K_t$ in the current direction $v_t$) by a constant factor. In particular, they show that this is achieved by choosing $y_t$ so to divide the expanded set $K_t + r\B$ exactly in half by volume. Since the Steiner potential at scale $r$ is bounded below by $\Vol(r\B)$, this allows the authors to bound the total number of mistakes at this scale. (A more detailed description of this algorithm is provided in \Cref{sect:csearch}). 

In the separation oracle setting, we do not know $v_t$ ahead of time, and thus cannot implement this algorithm as written. For example, we cannot guarantee our hyperplane splits $K_t + r\B$ exactly in half. We partially work around this by using (approximate variants of) Grunbaum's theorem, which guarantees that any hyperplane through the center-of-gravity of a convex set splits that convex set into two pieces of roughly comparable volume. In other words, everywhere where the contextual search algorithm divides the volume of $K_t + r\B$ in half, Grunbaum's theorem implies we obtain comparable results by choosing any hyperplane passing through the center-of-gravity of $K_t + r\B$.

Unfortunately, we still cannot quite implement this in the separation oracle setting, since the choice of $r$ in the contextual search algorithm depends on the input vector $v_t$. Nonetheless, by modifying the analysis of contextual search we can still get some guarantees via simple methods of this form. In particular we show that always querying the center-of-gravity of $K_t$ (alternatively, the center of the John ellipsoid of $K_t$) results in an $\exp(O(d\log d))$-regret cutting-plane algorithm, and that always querying the center of gravity of $K_t + \frac{1}{T}\B$ results in an $O(d\log T)$-regret cutting-plane algorithm.

Our cutting-plane algorithm for weak separation oracles requires a more nuanced understanding of the family of sets of the form $K_t + r\B$. This family of sets has a number of surprising algebraic properties. One such property (famous in convex geometry and used extensively in earlier algorithms for contextual search) is \textit{Steiner's formula}, which states that for any convex $K$, $\Vol(K + r\B)$ is actually a polynomial in $r$ with nonnegative coefficients. These coefficients are called \textit{intrinsic volumes} and capture various geometric measures of the set $K$ (including the volume and surface area of $K$).

There exists a lesser-known analogue of Steiner's formula for the center-of-gravity of $K + r\B$, which states that each coordinate of $\cg(K + r\B)$ is a rational function of degree at most $d$; in other words, the curve $\cg(K + r\B)$ for $r \in [0, \infty)$ is a rational curve. Moreover, this variant of Steiner's formula states that each point $\cg(K + r\B)$ can be written as a convex combination of $d+1$ points contained within $K$ known as the \textit{curvature centroids} of $K$. Motivated by this, we call the curve $\rho_K(r) = \cg(K + r\B)$ the \textit{curvature path} of $K$. 

Since the curvature path $\rho_K$ is both bounded in algebraic degree and bounded in space (having to lie within the convex hull of the curvature centers), we can bound the total length of the curvature path $\rho_K$ by a polynomial in $d$ (since it is bounded in degree, each component function of $\rho_K$ can switch from increasing to decreasing a bounded number of times). This means that we can discretize the curvature path to within precision $\eps$ while only using $\poly(d)/\eps$ points on the path. 

Our algorithms against weak separation oracles and for list contextual recommendation both make extensive use of such a discretization. For example, we show that in order to construct a low-regret algorithm against a weak separation oracle, it suffices to discretize $\rho_{K_t}$ into $O(d^4)$ points and then query a random point; with probability at least $O(d^{-4})$, we will closely enough approximate the point $\rho(r) = \cg(K + r\B)$ that our above analogue of contextual search would have queried. We show this results in $\poly(d)$ total regret\footnote{The reason this type of algorithm does not work against strong separation oracles is that each point in this discretization could return a different direction $v_t$, in turn corresponding to a different value of $r$}. A similar strategy works for list contextual recommendation: there we discretize the curvature path for the knowledge set $K_t$ into $\poly(d)$ candidate values for $w^*$, and then submit as our set of actions the best response for each of these candidates. 


\subsection{Related work}
There is a very large body of work on recommender systems which employs a wide range of different techniques -- for an overview, see the survey by Bobadilla et al. \cite{bobadilla2013recommender}. Our formulation in this paper is closest to treatments of recommender systems which formulate the problem as an online learning problem and attack it with tools such as contextual bandits or reinforcement learning. Some examples of such approaches can be seen in \cite{li2010contextual, li2011unbiased, tang2014ensemble, song2014online, warlop2018fighting}. Similarly, there is a wide variety of work on online shortest path routing \cite{zou2014online, awerbuch2004adaptive,gyorgy2006adaptive,gyorgy2007line, kveton2015combinatorial,talebi2017stochastic} which also applies tools from online learning. One major difference between these works and the setting we study in our paper is that these settings often rely on some quantitative feedback regarding the quality of item recommended. In contrast, our paper only relies on qualitative feedback of the form ``action $x$ is the best action this round'' or ``action $x$ is is at least as good as any action recommended''.

One setting in the bandits literature that also possesses qualitative feedback is the setting of Duelling Bandits~\cite{yue2012k}. In this model, the learner can submit a pair of actions and the feedback is a noisy bit signalling which action is better. However, their notion of regret (essentially, the probability the best arm would be preferred over the arms chosen by the learner) significantly differs from the notion of regret we measure in our setting (the loss to the user by following our recommendations instead of choosing the optimal actions).

Cutting-plane methods have a long and storied history in convex optimization. The very first efficient algorithms for linear programming (based on the ellipsoid method \cite{khachiyan1979polynomial,grotschel1981ellipsoid}). Since then, there has been much progress in designing more efficient cutting-plane methods (e.g.~\cite{bubeck-lee-singh}), but the focus remains on the number of calls to the separating oracle or the total running time of the algorithm. We are not aware of any work which studies cutting-plane methods under the notion of regret that we introduce in \Cref{sec:introlrcp}.

Contextual search was first introduced in the form described in \Cref{sect:csearch} in~\cite{paesleme2018contextual}, where the authors gave the first time-horizon-independent regret bound of $O(\poly(d))$ for this problem (earlier work by \cite{lobel2016multidimensional} and \cite{cohen2016feature} indirectly implied bounds of $O(\poly(d)\log T)$ for this problem). This was later improved by~\cite{liu2021optimal} to a near-optimal $O(d\log d)$ regret bound. The algorithms of both~\cite{paesleme2018contextual,liu2021optimal} rely on techniques from integral geometry, and specifically on understanding the intrinsic volumes and Steiner polynomial of the set of possible values for $w^*$. Some related geometric techniques have been used in recent work on the convex body chasing problem\cite{bubeck2020chasing,sellke2020chasing,argue2020chasing}. To our knowledge, our paper is the first paper to employ the fact that the curvature path $\cg(K + r\B)$ is a bounded rational curve (and thus can be efficiently discretized) in the development of algorithms.

\section{Model and preliminaries}

We begin by briefly reviewing the problems of contextual recommendation and designing low-regret cutting plane algorithms. In all of the below problems, $\ball = \{w \in \R^d \mid\, \Vert w \Vert_2 \leq 1\}$ is the ball of radius $1$ (and generally, all vectors we consider will be bounded to lie in this ball).

\paragraph{Contextual recommendation.} In \textit{contextual recommendation} there is a hidden point $w^* \in \ball$. Each round $t$ (for $T$ rounds) we are given a set of possible actions $\X_t \subseteq \ball$. If we choose action $x_t \in \X_t$ we obtain reward $\bkt{x_t}{w^*}$ (but do not learn this value). Our feedback is $x^*_t = \arg\max_{x \in \X_t}\bkt{x}{w^*}$, the identity of the best action\footnote{If this argmax is multi-valued, the adversary may arbitrarily return any element of this argmax.}. Our goal is to minimize the total expected regret $\E[\mathrm{Reg}] = \E\left[\sum_{t=1}^{T} \bkt{x^*_t - x_t}{w^*}\right]$. Note that since the feedback is deterministic, this expectation is only over the randomness of the learner's algorithm. 

It will be useful to establish some additional notation for discussing algorithms for contextual recommendation. We define the \textit{knowledge set} $K_t$ to be the set of possible values for $w^*$ given the knowledge we have obtained by round $t$. Note that the knowledge set $K_t$ is always convex, since the feedback we receive each round (that $\bkt{x^*}{w^*} \geq \bkt{x}{w^*}$ for all $x \in \X_t$) can be written as an intersection of several halfspaces (and the initial knowledge set $K_1 = \ball$ is convex). In fact, we can say more. Given a $w \in K_t$, let $$\BR_{t}(w) = \arg\max_{x \in \X_t}\bkt{x}{w}$$ be the set of optimal actions in $\X_t$ if the hidden point was $w$. We can then partition $K_t$ into several convex subregions based on the value of $\BR_{t}(w)$; specifically, let $$R_{t}(x) = \{w \in K_t | x \in \BR_{t}(w)\}$$ be the region of $K_t$ where $x$ is the optimal action to play in response. Then: 

\begin{enumerate}
    \item Each $R_{t}(x)$ is a convex subset of $K_t$.
    \item The regions $R_{t}(x)$ have disjoint interiors and partition $K_t$.
    \item $K_{t+1}$ will equal the region $R_{t}(x^*)$ (where $x^* \in \BR_{t}(w^*)$ is the optimal action returned as feedback).
\end{enumerate}

We also consider two other variants of contextual recommendation in this paper (\textit{list contextual recommendation} and \textit{local contextual recommendation}). We will formally define them as they arise (in Sections \ref{sec:polyd} and \ref{sec:loc} respectively).

\paragraph{Designing low-regret cutting-plane algorithms.} In a \textit{ low-regret cutting-plane algorithm}, we again have a hidden point $w^* \in \ball$. Each round $t$ (for $T$ rounds) we can query a separation oracle with a point $p_t$ in $\ball$. The separation oracle then provides us with an adversarially chosen direction $v_t$ (with $\norm{v_t} = 1$) that satisfies $\bkt{w^*}{v_t} \geq \bkt{p_t}{v_t}$. The regret in round $t$ is equal to $\bkt{w^* - p_t}{v_t}$, and our goal is to minimize the total expected regret $\E[\mathrm{Reg}] = \E\left[\sum_{t=1}^{T} \bkt{w^* - p_t}{v_t}\right]$. Again, since the feedback is deterministic, the expectation is only over the randomness of the learner's algorithm.

As with contextual recommendation, it will be useful to consider the knowledge set $K_t$, consisting of possibilities for $w^*$ which are still feasible by the beginning of round $t$. Again as with contextual recommendation, $K_t$ is always convex; here we intersect $K_t$ with the halfspace provided by the separation oracle every round (i.e. $K_{t+1} = K_{t} \cap \{\bkt{w - p_t}{v_t} \geq 0\}$). 

Unless otherwise specified, the separation oracle can arbitrarily choose $v_t$ as a function of the query point $p_t$. For obtaining low-regret algorithms for list contextual recommendation, it will be useful to consider a variant of this problem where the separation oracle must commit to $v_t$ (up to sign) at the beginning of round $t$. Specifically, at the beginning of round $t$ (before observing the query point $p_t$), the oracle fixes a direction $u_t$. Then, on query $p_t$, the separation oracle returns the direction $v_t = u_t$ if $\bkt{w - p_t}{u_t} \geq 0$, and the direction $v_t = -u_t$ otherwise. We call such a separation oracle a \textit{weak separation oracle}; an algorithm that only works against such separation oracles is a \textit{low-regret cutting-plane algorithm for weak separation oracles}. Note that this distinction only matters when the learner is using a randomized algorithm; if the learner is deterministic, the adversary can predict all the directions $v_t$ in advance.

\subsection{Convex geometry preliminaries and notation}

We will denote by $\Conv_d$ the collection of all convex bodies in $\R^d$. Given a convex body $K \in \Conv_d$, we will use $\Vol(K) = \int_K 1 dx$ to denote its volume (the  standard Lebesgue measure). Given two sets $K$ and $L$ in $\R^d$, their Minkowski sum is given by $K+L = \{x+y; x \in K, y \in L\}$.  Let $\B^d$ denote the unit ball in $\R^d$, let $\Sph^{d-1} = \{x \in \R^d; \Vert x \Vert_2 = 1\}$ denote the unit sphere in $\R^d$ and let $\kappa_d = \Vol(\B^d)$ be the volume of the $i$-th dimensional unit ball. When clear from context, we will omit the superscripts on $\B^d$ and $\Sph^{d-1}$. 

We will write $\cg(K) = (\int_K x dx) / (\int_K 1 dx)$ to denote the \emph{center of gravity} (alternatively, \emph{centroid}) of $K$. Given a direction $u \in \Sph^{d-1}$ and convex set $K \in \Conv_d$ we define the width of $K$ in the direction $u$ as:
$$\width(K;u) = \max_{x \in K} \bkt{u}{x} - \min_{x \in K} \bkt{u}{x}$$

\paragraph{Approximate Grunbaum and  John's Theorem} Finally, we state two fundamental theorems in convex geometry. Grunbaum's Theorem bounds the volume of the convex set in each side of a hyperplane passing through the centroid. For our purposes it will be also important to bound a cut that passes near, but not exactly at the centroid. The bound given in the following paragraph comes from a direct combination of Lemma B.4 and Lemma B.5 in~\citet{bubeck2020chasing}. 

We will use the notation $H_u(p) = \set{x \mid \bkt{x}{u} = \bkt{p}{u} }$ to denote
the halfspace passing through $p$ with normal vector $u$. Similarly, we let 
$H^{+}_{u}(p) = \set{ x \mid \bkt{x}{u} \geq \bkt{p}{u} } $. 

\begin{theorem}[Approximate Grunbaum \cite{bertsimas2004solving,bubeck2020chasing}]\label{thm:approx_grunbaum}
Let $K \in \Conv_d$, $c = \cg(K)$ and $u \in \Sph^{d-1}$. Then consider the semi-space $H_+ = \{x \in \R^d; \bkt{u}{x-c} \geq t\}$ for some $t \in \R_+$. Then:

$$\frac{\Vol(K \cap H_+)}{\Vol(K)} \geq \frac{1}{e} - \frac{2 t (d+1)}{\width(K;u)}$$
\end{theorem}

John's theorem shows that for any convex set $K \in \Conv_d$, we can find an ellipsoid $E$ contained in $K$ such that $K$ is contained in (some translate of) a dilation of $E$ by a factor of $d$.

\begin{theorem}[John's Theorem]\label{thm:john}
Given $K \in \Conv_d$, there is a point $q \in K$ and an invertible linear transformation $A : \R^d \rightarrow \R^d$ such that $$q + \B \subseteq A(K) \subseteq q + d \B.$$
\end{theorem}

We call the ellipsoid $E = A^{-1}(q+\B)$ in Theorem \ref{thm:john} the \textit{John ellipsoid} of $K$. 



\subsection{Contextual search}\label{sect:csearch}

In this section, we briefly sketch the algorithm and analysis of \cite{liu2021optimal} for the standard contextual search problem. We will never use this algorithm directly, but many pieces of the analysis will prove useful in our constructions of low-regret cutting-plane algorithms.

Recall that in contextual search, each round the learner is given a direction $v_t$. The learner is trying to learn the location of a hidden point $w^*$, and at time $t$ has narrowed down the possibilities of $w^*$ to a knowledge set $K_t$. The algorithm of \cite{liu2021optimal} runs the following steps:

\begin{enumerate}
    \item Compute the width $w = \width(K_t; v_t)$ of $K_t$ in the direction $v_t$. Let $r = 2^{\lceil \lg (w/10d) \rceil}$ (rounding $w/10d$ to a nearby power of two).
    \item Consider the set $\tilde{K} = K_t + r\B$. Choose $y_t$ so that the hyperplane $H = \{w \mid \bkt{v_t}{w} = y_t\}$ divides the set $\tilde{K}$ into two pieces of equal volume.
\end{enumerate}

We can understand this algorithm as follows. Classic cutting-plane methods try to decrease $\Vol(K_t)$ by a constant factor every round (arguing that this decrease can only happen so often before one of our hyperplanes passes within some small distance to our feasible region). The above algorithm can be thought of as a multi-scale variant of this approach: they show that if we incur loss $w \approx dr$ in a round (since loss in a round is at most the width), the potential function $\Vol(K_t + r\B)$ must decrease by a constant factor. Since $\Vol(K_t + r\B) \geq \Vol(r\B) = r^{d}\kappa_d$, we can incur a loss of this size at most $O(d \log (2/r))$ times. Summing over all possible discretized values of $r$ (i.e. powers of 2 less than 1), we arrive at an $O(d\log d)$ regret bound.

There is one important subtlety in the above argument: if we let $H^{+} = \{ w \mid \bkt{v_t}{w} \geq y_t\}$ be the halfspace defined by $H$, the two sets $(K_t \cap H^{+}) + r\B$ and $(K_t + r\B) \cap H^{+}$ are \textit{not} equal. The volume of the first set represents the new value of our potential (i.e. $\Vol(K_{t+1} + r\B)$), but it is the second set that has volume equal to half our current potential (i.e. $\frac{1}{2}\Vol(K_{t} + r\B)$).

Luckily, our choice of $r$ allows us to relate these two quantities in a way so that our original argument works. Let $H$ divide $K$ into $K^{+}$ and $K^{-}$. Note that $\Vol(K^{+} + r\B) + \Vol(K^{-} + r\B) = \Vol(K + r\B) + \Vol((K \cap H) + r\B)$ (in particular, $K + r\B$ and $(K \cap H) + r\B$ are the union and intersection respectively of $K^{+} + r\B$ and $K^{-} + r\B$). Since $\Vol(K^{+} + r\B) = \Vol(K^{-} + r\B)$, to bound $\Vol(K^{+} + r\B)/\Vol(K + r\B)$ it suffices to bound $\Vol((K \cap H) + r\B)$. We do so in the following lemma (which will also prove useful to us in later analysis).

\begin{lemma}\label{lem:conelem}
Given $K \in \Conv_d$ and $u \in \Sph^{d-1}$, let $H$ be a hyperplane of the form $\{w \mid \bkt{w}{u} = b\}$ (for some $b \in \R$). Then:
$$\Vol((K \cap H) + r\B) \leq \left(\frac{2rd}{\width(K; u)}\right) \cdot \Vol(K + r\B)$$
\end{lemma}
\begin{proof}
Let $\overline{V} = \Vol_{d-1}((K + r\B) \cap H)$ be the volume of the $(d-1)$-dimensional cross-section of $K + r\B$ carved out by $H$. Note first that we can write any point in $(K \cap H) + r\B$ in the form $w + \lambda u$, where $w \in (K + r\B) \cap H$ and $\lambda \in [-r, r]$. It follows that 

\begin{equation}\label{eq:conelem1}
    \Vol((K \cap H) + r\B) \leq 2r \overline{V}.
\end{equation}

We will now bound $\overline{V}$. Let $\overline{K} = (K + r\B) \cap H$. Let $p^{+}$ be the point in $K + r\B$ maximizing $\bkt{u}{p}$, and let $p^{-}$ be the point in $K + r\B$ minimizing $\bkt{u}{p}$ (so $p^{-}$ and $p^{+}$ certify the width). Consider the cones $C^{-}$ and $C^{+}$ formed by taking the convex hull $\Conv(p^{-}, \overline{K})$ and $\Conv(p^{+}, \overline{K})$ respectively. $C^{-}$ and $C^{+}$ are disjoint and contained within $K + r\B$, so
    
$$\Vol(C^{-}) + \Vol(C^{+}) \leq \Vol(K + r\B).$$

\noindent
But now note that by the formula for the volume of a cone, 

$$\Vol(C^{-}) + \Vol(C^{+}) = \frac{1}{d}\cdot\width(K+r\B; u)\cdot\Vol_{d-1}(\overline{K}) \geq \frac{\width(K; u)}{d}\cdot \overline{V}.$$

\noindent
It follows that 

\begin{equation}\label{eq:conelem2}
    \overline{V} \leq \frac{d}{\width(K; u)}\Vol(K + r\B).
\end{equation}

\noindent
Substituting this into \eqref{eq:conelem1}, we arrive at the theorem statement.
\end{proof}

This lemma allows us to conclude our analysis of the contextual search algorithm. In particular, since we have chosen $r \approx \width(K, v_t)/10d$, by applying this lemma we can see that in our analysis of contextual search, $\Vol((K \cap H) + r\B) \leq 0.2 \Vol(K + r\B)$, from which it follows that $\Vol(K^{+} + r\B)/\Vol(K + r\B) \leq 0.6$. 


\section{From Cutting-Plane Algorithms to Contextual Recommendation}

We begin by proving a reduction from designing low-regret cutting plane algorithms to contextual recommendation. Specifically, we will show that given a regret $\rho$ cutting-plane algorithm, we can use it to construct an $O(\rho)$-regret algorithm for contextual recommendation. 

Note that while these two problems are similar in many ways (e.g. they both involve searching for an unknown point $w^*$), they are not completely identical. Among other things, the formulation of regret although similar is qualitatively different between the two problems (i.e. between expressions \eqref{eq:contextual-selection-regret} and \eqref{eq:cutting-plane-reg}). In particular, in contextual recommendation, the regret each round is $\bkt{x^{*}_t - x_t}{ w^* }$, whereas for cutting-plane algorithms, the regret is given by $\bkt{w^* - p_t}{v_t}$. Nonetheless, we will be able to relate these two notions of regret by considering a separation oracle that always returns a halfspace in the direction of $x^{*}_t - x_t$. We present this reduction below.

\begin{theorem}\label{thm:reduction}
Given a low-regret cutting-plane algorithm $\mathcal{A}$ with regret $\rho$, we can construct an $O(\rho)$-regret algorithm for contextual recommendation. 
\end{theorem}
\begin{proof}
We will simultaneously run an instance of $\mathcal{A}$ with the same hidden vector $w^*$. Each round we will ask $\mathcal{A}$ for its query $p_t$ to the separation oracle. We will then compute a $x_{t} \in \BR_{t}(p_t)$ (recall that $\BR_{t}(w)$ is the optimal action to play if $w$ is the true hidden vector) and submit $x_{t}$ as our action for this round of contextual recommendation. We then receive feedback $x_t^* \in \BR_{t}(w^*)$. Consider the following two cases:

\paragraph{Case 1:} If $x_t^* = x_t$, then our contextual recommendation algorithm incurs zero regret since we successfully chose the optimal point. In this case we ignore this round for $\mathcal{A}$ (i.e. we reset its state to its state at the beginning of round $t$).

\paragraph{Case 2:} If $x_t^* \neq x_t$, let $v_t = (x_t^* - x_t)/\norm{x_t^* - x_t}$. We will return $v_t$ to $\mathcal{A}$ as the separation oracle's answer to query $p_t$. Note that this is a valid answer, since

\begin{equation}\label{eq:reduction_sep_oracle}
\bkt{w^* - p_t}{v_t} = \frac{1}{\norm{x_t^* - x_t}}\left(\bkt{w^*}{x_t^* - x_t} + \bkt{p_t}{x_t - x_t^*}\right) \geq \frac{1}{\norm{x_t^* - x_t}}\bkt{w^*}{x_t^* - x_t} .
\end{equation}


Here the final inequality holds since (by the definition of $\BR_t(p_t)$) $\bkt{p_t}{x_t} \geq \bkt{p_t}{x}$  for any $x \in \X_t$. The RHS of \eqref{eq:reduction_sep_oracle} is in turn larger than zero, since $\bkt{w^*}{x_t^*} \geq \bkt{w^*}{x}$ for any $x \in \X_t$ (and thus this is a valid answer to the separation oracle). Moreover, note that the regret we incur under contextual recommendation is exactly $\bkt{w^*}{ x_t^* - x_t}$, so by rearranging equation \eqref{eq:reduction_sep_oracle}, we have that:

$$\bkt{w^*}{x_t^* - x_t} \leq \norm{x_t^* - x_t}\bkt{w^* - p_t}{v_t} \leq 2\bkt{w^* - p_t}{v_t}.$$

It follows that the total regret of our algorithm for contextual recommendation is at most twice that of $\A$. Our regret is thus bounded above by $2\rho$, as desired.



\end{proof}

Note that the reduction in Theorem \ref{thm:reduction} is efficient as long as we have an efficient method for optimizing a linear function over $\X_t$ (i.e. for computing $\BR_{t}(w)$). In particular, this means that this reduction can be practical even in settings where $\X_t$ may be combinatorially large (e.g. the set of $s$-$t$ paths in some graph). 

Note also that this reduction \textit{does not} work if $\mathcal{A}$ is only low-regret against weak separation oracles. This is since the direction $v_t$ we choose does depend non-trivially on the point $p_t$ (in particular, we choose $x_t \in \BR_t(p_t)$). Later in~\Cref{sec:contextual-recomendation}, we will see how to use ideas from designing cutting-plane methods for weak separation oracles to construct low-regret algorithms for \textit{list} contextual recommendation -- however we do not have a black-box reduction in that case, and our construction will be more involved.

\section{Designing Low-Regret Cutting-Plane Algorithms}

In this section we will describe how to construct low-regret cutting-plane algorithms for strong separation oracles.

\subsection{An \texorpdfstring{ $\exp(O(d\log d))$}{exp(O(d log d))}-regret cutting-plane algorithm}

We begin with a quick proof that always querying the center of the John ellipsoid of $K_t$ leads to a $\exp(O(d\log d))$-regret cutting-plane algorithm. Interestingly, although this corresponds to the classical ellipsoid algorithm, our analysis will instead proceed along the lines of the analysis of the contextual search algorithm summarized in \Cref{sect:csearch}. 

We will need the following lemma.

\begin{lemma}\label{lem:expdratio}
Let $K \in \Conv_d$ be an arbitrary convex set and let $r \geq 0$. Let $E$ be the John ellipsoid of $K$, and let $H$ be a hyperplane that passes through the center of $E$, dividing $K$ into two regions $K^{+}$ and $K^{-}$. Then

$$\Vol(K^{+} + r\B) \leq \left(1 - \frac{1}{10d^d}\right)\left(\Vol(K^{+} + r\B) + \Vol(K^{-} + r\B)\right)$$
\end{lemma}
\begin{proof}
Let $H$ divide $E$ into the two regions $E^{+}$ and $E^{-}$ analogously to how it divides $K$ into $K^{+}$ and $K^{-}$. Note that since $E \subseteq K \subseteq dE$ (translating $K$ so that $E$ is centered at the origin), we can write:

\begin{equation}
    \frac{\Vol(K^{-} + r\B)}{\Vol(K + r\B)} \geq \frac{\Vol(E^{-} + r\B)}{\Vol(dE + r\B)} \geq \frac{0.5 \cdot \Vol(E + r\B)}{\Vol(dE + r\B)} \geq \frac{1}{2d^{d}}\frac{\Vol(E + r\B)}{\Vol(E + (r/d)\B)} \geq \frac{1}{2d^{d}}.
\end{equation}

On the other hand, by monotonicity we also have that

$$\frac{\Vol(K^{+} + r\B)}{\Vol(K + r\B)} \leq 1.$$

It follows that

$$\Vol(K^{+} + r\B)/\Vol(K^{-} + r\B) \leq 2d^{d}.$$

The conclusion then follows since 

$$2d^{d} \leq \left(1 - \frac{1}{10d^{d}}\right)(2d^{d} + 1).$$
\end{proof}

We can now modify the analysis of contextual search to make use of Lemma \ref{lem:expdratio}. In particular, we will show that for each round $t$, there's some $r$ (roughly proportional to the current width) where $\Vol(K_t + rB)$ decreases by a multiplicative factor of $(1 - d^{-O(d)})$. 

\begin{theorem}\label{thm:expdlogd}
The cutting-plane algorithm which always queries the center of the John ellipsoid of $K_t$ incurs $\exp(O(d \log d))$ regret.
\end{theorem}
\begin{proof}
Fix a round $t$, and let $K = K_t$ be the knowledge set at time $t$. Let $E$ be the John ellipsoid of $K$ and let $p_t$ be the center of $E$. When we query the separation oracle with $p_t$, we get a hyperplane $H$ (defined by $v_t$) that passes through $p_t$ and divides $K$ into $K^{+} = K_{t+1}$ and $K^{-} = K \setminus K_{t+1}$. 

By Lemma \ref{lem:expdratio}, for any $r \geq 0$, we have that 

$$\Vol(K^{+} + r\B) \leq \left(1 - \frac{1}{10d^d}\right)\left(\Vol(K^{+} + r\B) + \Vol(K^{-} + r\B)\right)$$

Note that (as in \Cref{sect:csearch}), $\Vol(K^{+} + r\B) + \Vol(K^{-} + r\B) = \Vol(K + r\B) + \Vol((K \cap H) + r\B)$. By Lemma \ref{lem:conelem}, we have that 

$$\Vol((K \cap H) + r\B) \leq \frac{2rd}{\width(K; v_t)}\cdot \Vol(K + r\B),$$

\noindent
and thus that

$$\Vol(K^{+} + r\B) \leq \left(1 - \frac{1}{10d^d}\right)\left(1 + \frac{2dr}{\width(K; v_t)}\right)\Vol(K + r\B)$$

\noindent
In particular, if we choose $r \leq \width(K; v_t)/(100d^{d+1})$, then 

$$\Vol(K^{+} + r\B) \leq \left(1 - \frac{1}{20d^d}\right)\Vol(K + r\B).$$

The analysis now proceeds as follows. In each round, let $r = 2^{\lfloor \lg (\width(K; v_t)/100d^{d+1}) \rfloor}$ be the largest power of $2$ smaller than $w/(100d^{d+1})$. Any specific $r$ can occur in at most

$$\frac{\log(\Vol(K_0 + r\B)/\Vol(K_{T} + r\B))}{\log\left(1 - \frac{1}{20d^{d}}\right)}$$

\noindent
rounds. This in turn is at most

$$\frac{\log(\Vol(2\B)/\Vol(r\B))}{1/(20d^d)} \leq 20d^{d+1}\log(2/r)$$

\noindent
rounds, and in each such round the regret that round is at most $\width(K;v_{t}) \leq 200d^{d+1}r$. The total regret from such rounds is therefore at most

$$20d^{d+1}\log(2/r) \cdot 200d^{d+1}r = O(d^{2(d+1)} r\log(2/r)).$$

\sloppy{
Now, by our discretization, $r$ is a power of two less than $1$. Note that 
$\sum_{i=0}^{\infty} 2^{-i}\log(2/2^{-i}) = O\left(\sum_{i=0}^{\infty} 2^{-i}i\right) = O(1)$. It follows that the total regret over all rounds is at most $O(d^{2(d+1)}) = \exp(O(d\log d))$, as desired.}
\end{proof}

The remaining algorithms we study will generally query the center-of-gravity of some convex set, as opposed to the center of the John ellipsoid. This leads to the following natural question: what is the regret of the cutting-plane algorithm which always queries the center-of-gravity of $K_t$?

Kannan, Lovasz, and Simonovits (Theorem 4.1 of \cite{kannan1995isoperimetric}) show that it is possible to choose an ellipsoid $E$ satisfying $E \subseteq K \subseteq dE$ such that $E$ is centered at $\cg(K)$, so our proof of Theorem \ref{thm:expdlogd} shows that this algorithm is also an $\exp(O(d\log d))$ algorithm. However, for both this algorithm and the ellipsoid algorithm of Theorem \ref{thm:expdlogd}, we have no non-trivial lower bound on the regret. It is an interesting open question to understand what regret these algorithms actually obtain (for example, do either of these algorithms achieve $\poly(d)$ regret?).

\subsection{An \texorpdfstring{$O(d\log T)$}{O(d log T)}-regret cutting-plane algorithm}

We will now show how to obtain an $O(d\log T)$-regret cutting plane algorithm. Our algorithm will simply query the center-of-gravity of $K_t + \frac{1}{T}\B$ each round. The advantage of doing this is that we will only need to examine one scale of the contextual search potential (namely the value of $\Vol(K_t + \frac{1}{T}\B)$). The following geometric lemma shows that, as long as the width of the $K_t$ is long enough, this potential decreases by a constant fraction each step.

\begin{lemma} \label{lem:volume-reduction}
Given $K \in \Conv_d$, $u \in \Sph^{d-1}$ and $b, r \in \R$ (with $r \geq 0$), let:

\begin{itemize}
    \item $c = \cg(K + r\B)$ be the center-of-gravity of $K + r\B$,
    \item $H^+(b) = \{\bkt{u}{x-c} \geq -b\}$ be a half-space induced by a hyperplane in the direction $u$ passing within distance $b$ of the point $c$, and
    \item $K^+ = K \cap H^+(b)$ be the intersection of $K$ with this half-space.
\end{itemize}
If $r, |b| \leq \width(K,u)/(16ed)$ then
$$\Vol(K^+ + r\B) \leq 0.9 \cdot \Vol(K+r\B).$$
\end{lemma}

\begin{proof}
Observe that $K^+ + rB \subseteq (K + rB) \cap H^+(b+r)$. If we define $H^-(b+r) = \{x \in \R^d; \bkt{u}{x-c} \leq -(b+r)\}$ then:
$$\Vol(K^+ + rB) \geq \Vol(K + rB) - \Vol((K + rB) \cap H^-(b+r) ).$$ By Theorem \ref{thm:approx_grunbaum} (Approximate Grunbaum) we have:
$$\frac{ \Vol((K + rB) \cap H^-(b+r) )}{\Vol(K + rB)} \geq \frac{1}{e} - \frac{2(d+1)}{\width(K;u)} \cdot \frac{2 \width(K,u)}{16ed} \geq \frac{1}{2e} \geq 0.1$$
\end{proof}

We can now prove that the above algorithm achieves $O(d\log T)$ regret.

\begin{theorem}\label{thm:dlogt}
The cutting-plane algorithm which queries the point $p_t = \cg\left(K_t + \frac{1}{T}\B\right)$ incurs $O(d\log T)$ regret. 
\end{theorem}
\begin{proof}
We will begin by showing that if we incur more than $50 d/T$ regret in a given round, we reduce the value of $\Vol(K_t + \frac{1}{T}\B)$ by a constant factor. Since $\Vol(K_t + \frac{1}{T}\B)$ is bounded below by $\Vol(\frac{1}{T}\B)$, this will allow us to bound the number of times we incur a large amount of regret.

Consider a fixed round $t$ of this algorithm. Let $K_t$ be the knowledge set at time $t$. When we query the separation-oracle point $p_t = \cg(K + \frac{1}{T}\B)$, we obtain a half-space $H^+ = \{w \in \R^d; \bkt{w - p}{v_t} \geq 0 \}$  passing through $p_t$ which contains $w^*$. We update $K_{t+1} = K_t \cap H^+$

The regret in round $t$ is bounded by $\width(K_t, v_t)$. If the width is at least  $50d/T$ we can then apply Lemma \ref{lem:volume-reduction} with $b=0$ and $r = 1/T$ to conclude that:
\begin{equation}\label{eq:dlogtconst}
    \Vol\left(K_{t+1} + \frac{1}{T}\B\right) \leq 0.9 \cdot \Vol\left(K_t + \frac{1}{T}\B\right).
\end{equation}

 $$\Vol\left(K^{+} + \frac{1}{T}\B\right) \leq \left(1 - \frac{1}{e} + 0.2\right) \Vol\left(K + \frac{1}{T}\B\right) < 0.9 \cdot \Vol\left(K + \frac{1}{T}\B\right).$$

Now, in each round where $\width(K_t, v_t) < 50d/T$, we incur at most $ 50d/T$ regret, so in total we incur at most $T \cdot (50d/T) = 50d$ regret from such rounds. On the other hand, in other rounds we may incur up to $\norm{w^{*} - p_t} \leq 2$ regret per round. However, note that $\Vol(K_1 + \frac{1}{T}\B) = \Vol((1 + \frac{1}{T})\B) \leq 2^d \Vol(B)$, whereas for any $t$, $\Vol(K_{t} + \frac{1}{T}\B) \geq \Vol(\frac{1}{T}\B) = T^{-d}\kappa_d$. Since in each such round we shrink this quantity by at least a factor of $0.9$, it follows that the total number of such rounds is at most $O(\log(2T^{d})) = O(d\log T).$
It follows that the total regret from such rounds is at most $O(d\log T)$, and thus the overall regret of this algorithm is at most $O(d \log T)$.
\end{proof}

\section{List contextual recommendation, weak separation oracles, and the curvature path}
\label{sec:polyd}

In this section, we present two algorithms: 1. a $\poly(d)$ expected regret cutting-plane algorithm for weak separation oracles, and 2. an $O(d^2\log d)$ regret algorithm for list contextual recommendation with list size $L = \poly(d)$. 

The unifying feature of both algorithms is that they both involve analyzing a geometric object we call the \textit{curvature path} of a convex body. The \textit{curvature path} of $K$ is a bounded-degree rational curve contained within $K$ that connects the center-of-gravity $\cg(K)$ with the Steiner point ($\lim_{r \rightarrow \infty} \cg(K + r\B)$) of $K$. 

In \Cref{sec:curvature_path} we formally define the curvature path and demonstrate how to bound its length. In \Cref{sec:weak-cutting-plane}, we show that randomly querying a point on a discretization of the curvature path leads to a $\poly(d)$ regret cutting-plane algorithm for weak separation oracles. Finally, in \Cref{sec:curvature_path}, we show how to transform a discretization of the curvature path of the knowledge set into a list of actions for list contextual recommendation, obtaining a low regret algorithm. 

\subsection{The curvature path}\label{sec:curvature_path}

An important fact (driving some of the recent results in contextual search, e.g. \cite{paesleme2018contextual}) is the fact that the volume $\Vol(K+r\B)$ is a $d$-dimensional polynomial in $r$. This fact is known as the Steiner formula:
\begin{equation}\label{eq:steiner}
   \Vol(K+r\B) = \sum_{i=0}^d V_{d-i}(K) \kappa_i r^i 
\end{equation}
After normalization by the volume of the unit ball, the coefficients of this polynomial correspond to the \emph{intrinsic volumes} of $K$. The intrinsic volumes are a family of $d+1$ functionals $V_i : \Conv_d \rightarrow \R_+$ for $i=0,1,\hdots, d$ that associate for each convex $K \in \Conv_d$ a non-negative value. Some of these functionals have natural interpretations: $V_d(K)$ is the standard volume $\Vol(K)$, $V_{d-1}(K)$ is the surface area, $V_1(K)$ is the average width and $V_0(K)$ is $1$ whenever $K$ is non-empty and $0$ otherwise. 

There is an analogue of the Steiner formula for the centroid of $K+r\B$, showing that it admits a description as a vector-valued rational function. More precisely, there exist $d+1$ functions $c_i : \Conv_d \rightarrow \R^d$ for $0 \leq i \leq d$ such that:
\begin{equation}\label{eq:curvature_path}
\cg(K + r\B) = \frac{\sum_{i=0}^d V_{d-i}(K) \kappa_i r^i\cdot c_i(K)}{\sum_{i=0}^d V_{d-i}(K) \kappa_i r^i }
\end{equation}

The point $c_0(K) \in K$ corresponds to the usual centroid $\cg(K)$ and $c_d(K)$ corresponds to the Steiner point. The functionals $c_i$ are called \emph{curvature centroids} since they can be computing by integrating a certain curvature measures associated with a convex body (a la Gauss-Bonnet). We refer to Section 5.4~ in Schneider \cite{schneider2014convex} for a more thorough discussion discussion. For our purposes, however, the only important fact will be that each curvature centroid $c_i(K)$ is guaranteed to lie within $K$ (note that this is not at all obvious from their definition).


Motivated by this, given a convex body $K \subseteq  \R^d$ we define its \emph{curvature path} to be the following curve in $\R^d$:
$$\rho_K : [0,\infty] \rightarrow K \qquad \rho_K(r) = \cg(K + r\B)$$
The path connects the centroid $\rho_K(0) = \cg(K)$ to the Steiner point $\rho_K(\infty)$. Our main result will exploit the fact that the coordinates of the curvature path are rational functions of bounded degree to produce a discretization. We start by bounding the length of the path. For reasons that will become clear, it will be more convenient to bound its length when transformed by the linear map in John's Theorem.

\begin{lemma} \label{lem:pathlen}
Let $K \in \Conv_d \setminus \{ \emptyset \}$, and let $A$ be a linear transformation as in (John's) Theorem \ref{thm:john}. Then the length of the path $\{A \rho_K(r); r \in [0,\infty]\}$ is at most $4 d^3$.
\end{lemma}

\begin{proof}
The length of a path is the integral of the $\ell_2$-norm of its derivative. We will bound the $\ell_2$ norm by the $\ell_1$ norm and then analyze each of its components.

\begin{equation}\label{eq:pathlen1}
\length(A \rho_K) = \int_0^\infty \Vert A \rho'_K(r) \Vert_2 dr \leq \int_0^\infty \Vert A \rho'_K(r) \Vert_1 dr = \sum_{i=1}^d \int_0^\infty \vert (A \rho'_{K}(r))_i \vert dr
\end{equation}
where $(A \rho'_{K}(r))_i$ is the $i$-th component of the vector $A \rho'_{K}(r)$. By equation \eqref{eq:curvature_path}, we know that there are degree-$d$ polynomials $p(r)$ and $q(r)$ such that $(A \rho'_{K}(r))_i = p(r) / q(r)$ where $q(r) > 0$ for all $r \geq 0$. Hence we can write its derivative as:
$(A \rho'_{K}(r))_i = (p'(r) q(r) - p(r) q'(r))/(q(r)^2)$ which can be re-written as $h(r)/q(r)^2$ for a polynomial $h(r)$ of degree at most $2d-1$. Now a polynomial of degree at most $k$ can change signs at most $k$ times. So we can partition $[0,\infty]$ into at most $2d$ intervals $I_1, \hdots, I_{2d}$ (some possibly empty) such that the sign of $(A \rho'_{K}(r))_i$ is the same within each region (treating zeros arbitrarily). If $I_j = [a_j, b_j]$, we can then write:
\begin{equation}\label{eq:pathlen2}
\int_0^\infty \vert (A \rho'_{K}(r))_i \vert dr = \sum_{j=1}^{2d} \int_{a_j}^{b_j} \vert (A \rho'_{K}(r))_i \vert = \sum_{i=1}^{2d}  \vert (A \rho_{K}(b_j))_i - (A \rho_{K}(a_j))_i \vert \leq 4d^2
\end{equation}
where the last step follows from John's theorem. Since $A( \rho_K)$ is in $A(K)$ which is contained in a ball of radius $d$, the distance between the $i$-coordinate of two points is at most $2d$. Equations \eqref{eq:pathlen1} and \eqref{eq:pathlen2} together imply the statement of the lemma. 
\end{proof}

\begin{lemma} \label{lem:curve-discretization}
Given $K \in \Conv_d$ and a discretization parameter $k$, there exists a set $D = \{p_0, p_1, \hdots, p_k\} \subset K$ such that for every $r$ there is a point $p_i \in D$ such that:
$$\abs{\bkt{\rho_K(r) - p_i}{u}} \leq \frac{4d^3}{k} \cdot \width(K,u),\; \forall u \in \Sph^{d-1}.$$
\end{lemma}

\begin{proof}
Discretize the path $A \rho_k$ into $k$ pieces of equal length and let $A p_0, A p_1, \hdots, A p_k$ correspond to the endpoints. Let $D = \{p_0, p_1, \hdots, p_k\}$. We know by \Cref{lem:pathlen} that for any $p= \rho_K(r)$, there exists a $p_i \in D$ such that: $\norm{Ap_i - A p}_2 \leq 4d^3/k$.

Now, for each unit vector $u \in \Sph^{d-1}$, we have:
\[
\abs{\bkt{u}{ p_i - p}} \leq \bkt{A^{-T} u}{ A(p_i - p)} \leq \norm{ A^{-T} u } \cdot 
\norm{ A(p_i - p) } \leq \norm{ A^{-T} u } \cdot 4d^3/k \]

Finally, we argue that $\norm{ A^{-T} u } \leq   \width(K;u)$. 
Let $v = (A^{-T} u) / \norm{ A^{-T} u } $ and take $x,y \in K$ that  certify the 
width of $K$ in direction $u$: 
$$\width(K,u) = \bkt{u}{x-y} = \bkt{A^{-T} u}{ Ax - Ay} = \norm{ A^{-T} u } \cdot \bkt{v}{ Ax - Ay}
$$

Finally note that $Ax$ and $Ay$ are respectively the maximizer and minimizer of $\bkt{v}{z}$ for $z \in A(K)$ since: $\max_{z \in A(K)} \bkt{v}{z} = \max_{x \in K} \bkt{v}{Ax} = \max_{x \in K} \bkt{A^T v}{x} = \max_{x \in K} \bkt{u }{x} / \norm{A^{-T}u}$. This implies that $\bkt{v}{Ax - Ay} = \width(A(K),v) \geq 1$ by John's Theorem since $q + \B \subseteq A(K)$. This completes the proof.
\end{proof}

\subsection{Low-regret cutting-plane algorithms for weak separation oracles}\label{sec:weak-cutting-plane}

In this section we show how to use the discretization of the curvature path in Lemma \ref{lem:curve-discretization} to construct a $\poly(d)$-regret cutting-plane algorithm that works against a weak separation oracle.

Recall that a weak separation oracle is a separation oracle that fixes the direction of the output hyperplane in advance (up to sign). That is, at the beginning of round $t$ the oracle fixes some direction $v_t \in \Sph^{d-1}$ and returns either $v_t$ or $-v_t$ to the learner depending on the learner's choice of query point $q_t$. 

One advantage of working with a weak separation oracle is that the width $\width(K_t; v_t)$ of the knowledge set in the direction $v_t$ is fixed and independent of the query point $p_t$ of the learner. This means that if we can guess the width, we can run essentially the standard contextual search algorithm (of \Cref{sect:csearch}) by querying any point $p_t$ that lies on the hyperplane which decreases the potential corresponding to this width by a constant factor. One good way to guess the width turns out to choose a random point belonging to a suitably fine discretization of the curvature path. 

\begin{theorem}\label{thm:weak-cutting-plane}
The cutting-plane algorithm which chooses a random point from the discretization of the curvature path of $K_t$ into $d^4$ pieces achieves a total regret of $O(d^5\log^2 d)$ against any weak separation oracle.
\end{theorem}
\begin{proof}
Consider a fixed round $t$. Let $v_t$ be the direction fixed by the weak separation-oracle and let $\omega = \width(K_t;v_t)$. Let $r = 2^{\lceil \lg(\omega/16ed) \rceil}$ (rounding $\omega/16ed$ to the nearest power of two). 

If we could choose the point $p_t = \rho_{K_t}(r) = \cg(K_t + r\B)$, then by \Cref{lem:volume-reduction}, any separating hyperplane through $p_t$ would decrease this potential by a constant factor. However, we do not know $r$. Instead, we will choose a random point from the discretization $D$ of the curvature path of $K_t$ into $O(d^4)$ pieces, and argue that by Lemma \ref{lem:curve-discretization} one of these points will be close enough to $\rho_{K_t}(r)$ to make the argument go through. 

Formally, let $D$ be the discretization of $\rho_{K_t}$ into $64ed^4$ pieces as per \Cref{lem:curve-discretization}. By \Cref{lem:curve-discretization}, there then exists a point $p_i \in D$ that satisfies

\begin{equation}\label{eq:discretization_bound}
|\langle \rho_K(r) - p_i, v_t \rangle| \leq \frac{1}{16ed}\cdot\width(K_t;v_t).
\end{equation}

Let $H$ be a hyperplane through $p_i$ in the direction $v_t$ (i.e. $H = \{\bkt{w - p_i}{v_t} = 0\}$), and let $H$ divide $K_t$ into the two regions $K^{+}$ and $K^{-}$. By \Cref{lem:volume-reduction} (with $b = \bkt{\rho_{K}(r) - p_i}{v_t}$), since \eqref{eq:discretization_bound} holds, we have that

\begin{equation}\label{eq:wcp-vol-reduction}
    \Vol(K^+ + r\B) \leq 0.9 \cdot \Vol(K+r\B).
\end{equation}

Now, consider the algorithm which queries a random point in $D$. With probability $1/|D| = \Omega(d^{-4})$, equation \eqref{eq:wcp-vol-reduction} holds. Otherwise, it is still true that $\Vol(K^+ + r\B) \leq \cdot \Vol(K+r\B)$. Therefore in expectation, 

$$\E[\Vol(K_{t+1} + r\B)] \leq \left(1 - \Omega(d^{-4})\right)\E[\Vol(K_t + r\B)].$$

In particular, the total expected number of rounds we can have where $r = 2^{-i}$ is at most $di/\log(1/(1 - \Omega(d^{-4}))) = O(id^{5})$. In such a round, our maximum possible loss is at most $\width(K_t; v_t) \leq \min(20dr, 2)$. Summing over all $i$ from $0$ to $\infty$, we arrive at a total regret bound of

$$\sum_{i=0}^{\infty} O(id^{5}\min(d2^{-i}, 1)) = \sum_{i=0}^{\log d} O(id^{5}) + d^{6}\sum_{i=\log d}^{\infty} O(i2^{-i}) = O(d^{5}\log^2 d).$$
\end{proof}


\subsection{List contextual recommendation} \label{sec:contextual-recomendation}
In this section, we consider the problem of list contextual recommendation. In this variant of contextual recommendation, we are allowed to offer a list of possible actions
$L_t \subseteq \X_t$ and we measure regret against the best action in the list:
\[ \loss_t =  \bkt{w^*}{x_t^* } - \max_{x \in L_t} \bkt{w^*}{x}.\] 
Our main result is that if the list is allowed to be of size $O(d^4)$ then it is possible to achieve total regret $O(d^2 \log d)$.

The recommended list of actions will be computed as follows: given the knowledge set $K_t$, let $D$ be the discretization of the curvature path with parameter $k = 200 d^4$ obtained in Lemma \ref{lem:curve-discretization}. Then for each $p_i \in D$ find an arbitrary $x_i \in \BR(p_i) := \argmax_{x \in \X_t} \bkt{p_i}{x} $ and let $L_t = \{x_1, x_2, \hdots, x_k\}$.


\begin{theorem} \label{thm:low-regret-list}
There exists an algorithm which plays the list $L_t$ defined above and incurs a total regret of at most $O(d^2 \log d)$. 
\end{theorem}
\begin{proof}

The overall structure of the proof will be as follows: we will show that for each integer $j \geq 0$, the algorithm can incur loss  between $100d \cdot 2^{-j}$ and  $200d \cdot 2^{-j}$ at most $O(jd)$ times. Hence the total loss of the algorithm can be bounded by $\sum_{j=1}^\infty O(jd) \cdot 2^{-j} d \leq O(d^2 \log d)$.\\

\emph{Potential function:} This will be done via a potential function argument. As usual, we will keep track of knowledge $K_t$ which corresponds to all possible values of $w$ that are consistent with the observations seen so far. $K_1 = \B$ and:
$$K_{t+1} = K_t \cap \left[\cap_{i \in L_t} \{w \in \R^d; \bkt{x^* - x}{w} \geq 0\} \right]$$
Associated with $K_t$ we will keep track of a family of potential functions:
$$\Phi_t^j = \Vol(K_t +  2^{-j} \B)$$
Since $K_1 \supseteq K_2  \supseteq K_3  \supseteq ...$ the potentials will be non-increasing: $\Phi_1^j \geq \Phi_2^j \geq \Phi_3^j \geq ...$. One other important property is that the potential functions are lower bounded:
\begin{equation}\label{eq:potential_lower_bound}
\Phi_j^t \geq \Vol(2^{-j}\B) = 2^{-jd} \Vol(\B)
\end{equation}
We will argue that if we can bound the loss at any given step $t$ by $200 \cdot 2^{-j}d$, then $\Phi^j_{t+1} \leq 0.9 \cdot \Phi^j_t$. Because of the lower bound in equation \ref{eq:potential_lower_bound}, this can happen at most $$O\left(\log\left(\frac{\Phi_j^1 }{ 2^{-jd} \Vol(B)}\right)\right) = O\left(\log\left(\frac{(1+2^{-j})^d\Vol(\B)}{ 2^{-jd} \Vol(B)}\right)\right) \leq O(jd)$$

\emph{Bounding the loss:} We start by bounding the loss and depending on the loss we will show a constant decrease in a corresponding potential function. Let
$$x^*\in \argmax_{x \in \X_t} \bkt{w^*}{x}$$
If $x^*$ is in the convex hull of $L_t$ then there must some of the points in $x_i \in L_t$ that is also optimal, in which case the algorithm incurs zero loss in this round and we can ignore it. Otherwise, we can assume that $x^*$ is not in the convex hull of $L_t$.

In that case, define for each $x_i \in L_t$ the vector:
$$v_i = \frac{x^* - x_i}{\norm{x^*-x_i}_2}$$
Consider the index $i$ that minimizes $\width(K; v_i)$ and use this point to bound the loss:
$$\begin{aligned}
\loss_t & = \min_{x \in L_t} \bkt{w^*}{x^*-x} \leq \bkt{w^*}{x^*-x_i} \leq \bkt{w^* - p_i}{x^* - x_i}  \\ & = \bkt{w^* - p_i}{v_i} \cdot \norm{x^* - x_i}  \leq 2\bkt{w^* -p_i}{v_i} \leq 2 \width(K, v_i)
\end{aligned}$$
The second inequality above follows from the definition of $x_i$ since $x_i \in \argmax_{x \in \X_t} \bkt{p_i}{x}$ it follows that $\bkt{p_i}{x_i - x^*} \geq 0$.\\

\emph{Charging the loss to the potential} We will now charge this loss to the potential. For that we first define an index $j$ such that:
$$j = - \left\lceil \frac{\width(K,v_i)}{100 d} \right\rceil$$
With this definition we have:
$$\loss_t \leq 2 \width(K,v_i) \leq 200 d 2^{-j}$$
Our final step is to show that the potential $\Phi_t^j$ decreases by a constant factor. For that we will use a combination of the discretization in Theorem \ref{lem:curve-discretization} and the volume reduction guarantee in Lemma \ref{lem:volume-reduction}.

First consider the point:
$$g_i = \cg(K + 2^{-j} \B)$$
Since it is on the curvature path, there is a discretized point $p_\ell \in D$ such that:
$$\vert \bkt{v_\ell}{g_i - p_\ell} \vert \leq \width(K, v_\ell)/(50 d)$$
Together with the facts that $\bkt{w^*}{v_\ell} \geq 0$ and $\bkt{p_\ell}{v_\ell} \leq 0$ we obtain that:
$$\bkt{w^* - g_i}{v_\ell} = \bkt{w^* - p_\ell}{v_\ell} + \bkt{ p_\ell - g_i}{v_\ell} \geq -\width(K, v_\ell)/(50d)$$
This in particular implies that:
$$K_{t+1} \subseteq \tilde{K}_{t+1} := K_t \cap \{w \in \R^d; \bkt{w - g_i}{v_\ell} \geq  -\width(K, v_\ell)/(50d)\}$$
We are now in the position of applying \Cref{lem:volume-reduction} with $r = 2^{-j}$. Note that $$r = 2^{-j} \leq \frac{ \width(K, v_i)}{50 d} \leq \frac{\width(K, v_\ell)}{50 d}$$
where the last inequality follows from the choice of the index $i$ as the one minimizing $\width(K, v_i)$. Applying the Theorem, we obtain that:
$$\Vol(K_{t+1} + 2^{-j}\B) \leq \Vol(\tilde{K}_{t+1} + 2^{-j}\B) \leq 0.9 \cdot \Vol(K_t + 2^{-j}\B)$$
which is the desired decrease in the $\Phi^j_t$ potential. This concludes the proof.
\end{proof}

\newcommand{\loc}{\text{loc}}
\newcommand{\cD}{\mathcal{D}}
\newcommand{\cA}{\mathcal{A}}
\newcommand{\N}{\mathbb{N}}

\section{Local Contextual Recommendation}
\label{sec:loc}

In this section, we consider the \emph{local contextual recommendation} problem, in which we may choose a list of actions $L_t \subseteq \X_t$ and our feedback is some $x^{\loc}_t$ such that $\left<x^{\loc}_t, w^*\right> \geq \max_{x \in L_t} \left<x, w^*\right>$. In other words, the feedback may not be the optimal action but it must at least be as good as the local optimum in $L_t$. The goal is the same as before: minimize the total expected regret $\E[\mathrm{Reg}] = \E\left[\sum_{t=1}^{T} \bkt{x^*_t - x_t}{w^*}\right]$ where $x^*_t \in \arg\max_{x \in \X_t}\bkt{x}{w^*}$.

It should be noted that, in this model, it is impossible to achieve non-trivial regret if the list size $|L_t|$ is only one, since the feedback will always be the unique element, providing no information at all. Below we show that it is possible to achieve bounded regret algorithm even when $|L_t| = 2$, although the regret does depend on the total number of possible actions each round, i.e. $\max_t |\X_t|$. Furthermore, we show that, even when $|L_t|$ is allowed to be as large as $2^{\Omega(d)}$, the expected regret of any algorithm remains at least $2^{\Omega(d)}$. 

\subsection{Low-regret algorithms}
\label{sec:loc-alg}

We use $[a]_+$ as a shorthand for $\max\{a, 0\}$.

Our algorithm employs a reduction similar to that of \Cref{thm:reduction}. Specifically, we prove the following:
\begin{theorem}\label{thm:reduction-local}
Suppose that $|\X_t| \leq A$ for all $t \in \mathbb{N}$, and let $H$ be any positive integer such that $2 \leq H \leq A$. Then, given a low-regret cutting-plane algorithm $\mathcal{A}$ with regret $\rho$, we can construct an $O(\rho \cdot A / (H - 1))$-regret algorithm for local contextual recommendation where the list size $|L_t|$ in each step is at most $H$. 
\end{theorem}

Before we prove \Cref{thm:reduction-local}, notice that it can be combined with~\Cref{thm:expdlogd} and \Cref{thm:dlogt} respectively to yield the following algorithms for local contextual recommendation.

\begin{corollary} \label{cor:local-expdlod}
Suppose that $|\X_t| \leq A$ for all $t \in \mathbb{N}$, and let $H$ be any positive integer such that $2 \leq H \leq A$. Then, there is an $O\left(A / (H - 1) \cdot \exp(d \log d)\right)$-regret algorithm for local contextual recommendation where the list size $|L_t|$ in each step is at most $H$. 
\end{corollary}

\begin{corollary} \label{cor:local-dlogt}
Suppose that $|\X_t| \leq A$ for all $t \in \mathbb{N}$, and let $H$ be any positive integer such that $2 \leq H \leq A$. Then, there is an $O(A / (H - 1) \cdot d\log T)$-regret algorithm for local contextual recommendation where the list size $|L_t|$ in each step is at most $H$. 
\end{corollary}

Note that these algorithms work for list sizes as small as $H = 2$ but may also give a better regret bound if we allow larger lists.

We will now prove \Cref{thm:reduction-local}.

\begin{proof}[Proof of \Cref{thm:reduction-local}]
Our algorithm is similar to that of~\Cref{thm:reduction}, except that we also play $H - 1$ random actions from $\X_t$ in addition to the action determined by the answer of $\cA$. More formally, each round $t$ of our algorithm works as follows:
\begin{itemize}
\item Ask $\cA$ for its query $p_t$ to the separation oracle.
\item Let $x_t = \BR_t(p_t)$, and let $L'_t \subseteq \X_t$ be a random subset of $\X_t$ of size $\min\{H - 1, |\X_t|\}$.
\item Output the list $L_t = \{x_t\} \cup L'_t$. 
\item Let $x^{\loc}_t$ be the feedback.
\item If $x^{\loc}_t \ne x_t$, do the following:
\begin{itemize}
\item Return $v_t = (x^{\loc}_t - x_t)/\norm{x^{\loc}_t - x_t}$ to $\cA$.
\item Update the knowledge set $K_{t + 1} = \{w \in K_t \mid \left<x^{\loc}_t - x_t, w\right> \geq 0\}$.
\end{itemize}
\end{itemize}

We will now show that the expected regret of the algorithm is at most $\rho \cdot A / (H - 1)$. From the regret bound of $\cA$, the following holds regardless of the randomness of our algorithm:
\begin{align*} 
\rho \geq 
\sum_{t: x^{\loc}_t \ne x_t} \left<\frac{x^{\loc}_t - x_t}{\|x^{\loc}_t - x_t\|}, w^* - p_t \right>
&\geq \sum_{t: x^{\loc}_t \ne x_t} 0.5\left<x^{\loc}_t - x_t, w^* - p_t \right> \\
&= 0.5\left(\sum_{t} \left<x^{\loc}_t - x_t, w^* - p_t \right>\right).
\end{align*}

From the requirement of $x^{\loc}_t$, we may further bound $\left<x^{\loc}_t - x_t, w^* - p_t\right>$ by 
\begin{align*}
\left<x^{\loc}_t - x_t, w^* - p_t \right>
\geq \max_{x \in L_t} \left<x - x_t, w^* - p_t\right>
= \max_{x' \in L'_t} [\left<x' - x_t, w^* - p_t\right>]_+.
\end{align*}
Hence, from the above two inequalities, we arrive at
\begin{align*}
2\rho \geq \sum_{t} \max_{x' \in L'_t} [\left<x' - x_t, w^* - p_t\right>]_+.
\end{align*}
Next, observe that 
\begin{align*}
\E\left[\max_{x' \in L'_t} [\left<x' - x_t, w^* - p_t\right>]_+\right] &\geq \Pr[x^*_t \in L'_t] \cdot  \left<x^* - x_t, w^* - p_t\right> \\
&= \frac{|L'_t|}{|\X_t|} \cdot  \left<x^* - x_t, w^* - p_t\right> \\
&\geq \frac{H - 1}{A} \cdot \left<x^* - x_t, w^* - p_t\right>.
\end{align*}
Combining the above two inequalities, we get
\begin{align*}
2\rho \geq \frac{H - 1}{A} \cdot \E\left[\sum_{t} \left<x^*_t - x_t, w^*\right>\right].
\end{align*}
From this, we can conclude that the expected regret, which is equal to $\E\left[\sum_{t} \left<x^*_t - x_t, w^*\right>\right]$, is at most $O\left(\rho \cdot A / (H - 1)\right)$ as desired.
\end{proof}

\subsection{Lower Bound}
\label{sec:loc-lb}

We will now prove our lower bound. The overall idea of the construction is simple: we provide an action set that contains a ``reasonably good'' (publicly known) action so that, unless the optimum is selected in the list, the adversary can return this reasonably good action, resulting in the algorithm not learning any new information at all.

\begin{theorem}\label{thm:locallb}
Any algorithm for the local contextual recommendation problem that can output a list of size up to $2^{\Omega(d)}$ in each step incurs expected regret of at least $2^{\Omega(d)}$.
\end{theorem}

\begin{proof}
Let $S$ be any maximal set of vectors in $B_d$ such that the first coordinate is zero and the inner product between any pair of them is at most $0.1$. By standard volume argument, we have $|S| \geq 2^{\Omega(d)}$. Furthermore, let $e_1$ be the first vector in the standard basis. Consider the adversary that picks $u \in S$ uniformly at random and let $w^* = 0.2 e_1 + 0.8u$ and let $X_t = S \cup \{e_1\}$ for all $t \in \N$. The adversary feedback is as follows: if $u \notin L_t$, return $e_1$; otherwise, return $u$.

We will now argue that any algorithm occurs expected regret at least $2^{\Omega(d)}$, even when allows to output a list $L_t$ of size as large as $\lfloor \sqrt{|S|} \rfloor = 2^{\Omega(d)}$ in each step. From Yao's minimax principle, it suffices to consider only any deterministic algorithm $\cA$. Let $L^0_t$ denote the list output by $\cA$ at step $t$ if it had received feedback $e_1$ in all previous steps.

Observe also that in each step for which $u \notin L_t$, the loss of $\cA$ is at least 0.6. Furthermore, in the first $m = \lfloor 0.1 \sqrt{\abs{S}} \rfloor$ rounds, the probability that the algorithm selects $u$ in any list is at most $\frac{m \sqrt{\abs{S}}}{\abs{S}} \leq 0.1$. Hence we can bound the the expected total regret of $\cA$ as:
\begin{align*}
\E[0.6 \cdot |\{t \mid u \notin L_t\}|] \geq 0.6 m \Pr[u \notin \cup_{t=1}^m L_t] = 0.6 m \Pr[u \notin \cup_{t=1}^m L_t^0] \geq 0.6 m \cdot 0.9 \geq 2^{\Omega(d)}
\end{align*}

which concludes our proof.
\end{proof}

\bibliographystyle{plainnat}
\bibliography{main.bib}

\end{document}